\documentclass[10pt,twocolumn,letterpaper]{article}
\usepackage[letterpaper,margin=0.75in]{geometry}
\usepackage{graphicx}
\usepackage{booktabs}
\usepackage{amsmath}
\usepackage{amssymb}
\usepackage{algorithm}
\usepackage{algorithmic}
\usepackage{newfloat}
\usepackage{listings}
\usepackage[round]{natbib}
\usepackage[hyphens]{url}
\usepackage[hidelinks]{hyperref}
\usepackage{caption}
\usepackage{flushend}
\DeclareCaptionStyle{ruled}{labelfont=normalfont,labelsep=colon,strut=off}
\floatstyle{ruled}
\newfloat{listing}{tb}{lst}{}
\floatname{listing}{Listing}
\newtheorem{proposition}{Proposition}
\newtheorem{remark}{Remark}
\title{\bf Calibration, Not Compilation: Detecting and Repairing\\
Misspecified Probabilistic Programs Written by Language Models}
\author{
  Jian Xu$^{1,2}$ \quad Delu Zeng$^{3}$ \quad John Paisley$^{4}$ \quad Qibin Zhao$^{2}$\\[3pt]
  {\normalsize $^{1}$RIKEN iTHEMS \quad $^{2}$RIKEN AIP \quad
  $^{3}$South China University of Technology \quad $^{4}$Columbia University}\\[1pt]
  {\normalsize \texttt{jian.xu@riken.jp}}
}
\date{}

\begin{document}
\twocolumn[
\begin{@twocolumnfalse}
\maketitle
\vspace{-1.5em}
\begin{abstract}
Language models increasingly write probabilistic programs (in NumPyro, Stan, or Pyro), but a program that compiles, runs, and passes every unit test can still be \emph{statistically} wrong---a Gaussian likelihood for heavy-tailed data, a Poisson for over-dispersed counts, an invalid prior support, or a pathological parameterization. The right verifier is therefore not a test suite but the \emph{Bayesian workflow itself}: posterior predictive checks, simulation-based calibration, sampler diagnostics ($\hat R$, divergences, ESS), and held-out predictive density. We study this \emph{calibration oracle} along three axes. \textbf{Detection:} on a benchmark of $14$ misspecification types across $10$ model families ($200$ instances), it flags the bug with AUC $0.97$ ($88\%$ at $2\%$ FPR \emph{when handed the correct reference program, an upper bound})---and a fully \emph{reference-free} version that uses no correct program reaches $62$--$78\%$ (the upper figure from a small automated model search), versus $0\%$ for a unit-test oracle. \textbf{Repair:} used as feedback in an LLM repair loop across fifteen models, calibration significantly outperforms unit-test feedback---which is itself \emph{significantly worse than no feedback at all}, a passing test inducing false confidence that suppresses repair---and improves over no feedback on strong-but-unsaturated models (GPT-5.1 $33{\to}92\%$, Claude $75{\to}100\%$; paired McNemar, $n{=}228$). \textbf{Reality:} on programs LLMs write from scratch for neutral briefs, $15$--$47\%$ of runnable ones are statistically misspecified (unit tests catch none), and calibration-guided repair significantly beats LLM-as-judge review, a Bayesian-workflow checklist, and data-summary self-debug. Across all three, the lesson is the same: for probabilistic programs, correctness is calibration, not compilation.
\end{abstract}
\vspace{1em}
\end{@twocolumnfalse}
]

\section{Introduction}
Probabilistic programming languages (PPLs) such as Stan~\citep{carpenter2017stan}, Pyro/NumPyro~\citep{bingham2019pyro,phan2019composable}, and PyMC let practitioners specify a generative model and obtain a posterior by general-purpose inference. Writing a correct probabilistic program, however, requires statistical expertise: one must choose an appropriate likelihood, give parameters valid support, capture the dispersion and tail behavior of the data, and pick a parameterization that the sampler can actually explore. Large language models are increasingly used to draft such programs, and a tempting deployment pattern is to wrap the LLM in an agentic loop that ``tests'' its output and iterates.

The difficulty is that the usual notion of a test does not apply. A probabilistic program is not correct because it runs; it is correct because its posterior is well-calibrated and its predictions match held-out data. A model that uses a Gaussian likelihood for heavy-tailed data, a Poisson likelihood for over-dispersed counts, or a centered parameterization for a hierarchical funnel will compile, run under NUTS, return finite samples, and pass every conventional unit test---while being statistically wrong in a way that no amount of code inspection on a neutral task description can reveal. We call such errors \emph{code-invisible} misspecification.

Our central claim is that the right verifier for LLM-written probabilistic programs is the \emph{Bayesian workflow}~\citep{gelman2020bayesian}: simulation-based calibration (SBC)~\citep{talts2018validating}, posterior predictive checks (PPC)~\citep{gelman1996posterior}, sampler diagnostics such as $\hat R$~\citep{vehtari2021rank} and divergent transitions~\citep{hoffman2014no}, and held-out predictive density. These quantities form a \emph{calibration oracle} that, unlike a test suite, is sensitive precisely to statistical error. We study it as both a \emph{detector} (does it flag misspecification a test misses?) and a \emph{repair signal} (does feeding it back to the writer fix the program?), on a controlled benchmark and on programs models write from scratch. In one phrase: for probabilistic programs, correctness is calibration, not compilation.

We make three contributions.
\begin{itemize}
\item \textbf{A misspecification benchmark and the detection result.} We build a benchmark of $14$ misspecification types across $10$ model families ($200$ instances), all of which compile and run. The calibration oracle flags the bug in $88\%$ of cases ($93\%$ of code-invisible ones) at a $2\%$ false-positive rate (detection AUC $0.97$), versus $0\%$ for a unit-test oracle. Each error type lights up the \emph{right} diagnostic, and the oracle further \emph{localizes} the bug type with $75\%$ top-2 accuracy.
\item \textbf{It does not need a correct program.} A \emph{reference-free} oracle (PPC + sampler + SBC) detects $62\%$ of bugs; adding a classical GLM baseline reaches $68\%$, and an automated model search (LOO over a small standard library) reaches $78\%$---all without a hand-written correct program. The $88\%$ ``reference'' figure is an upper bound that assumes the correct program is known. An ablation shows held-out density is the strongest single component, while SBC alone cannot detect model--data misspecification.
\item \textbf{Calibration-as-oracle repair.} Used as feedback in an LLM repair loop and compared against no feedback (self-refine) and unit-test feedback, calibration is the strongest signal at every capability level able to use it, lifting GPT-5.1 from $33\%\!\to\!92\%$ and Claude Sonnet~4.6 from $75\%\!\to\!100\%$ on code-invisible bugs across fifteen models.
\item \textbf{Two robust phenomena.} (a)~Unit-test feedback is \emph{harmful}: ``all checks pass'' suppresses repair, so test feedback is $\le$ no feedback (Claude: $25\%$ vs.\ $75\%$). (b)~The benefit of calibration feedback is \emph{non-monotone} in capability---negligible for models too weak to act on a diagnostic or strong enough to repair blindly, and largest for strong-but-unsaturated models (GPT-5.1, Claude).
\item \textbf{Real, not just injected, failures.} On programs LLMs write from scratch for neutral briefs, $15$--$47\%$ of runnable ones are statistically misspecified (unit tests flag none), and calibration repair significantly beats LLM-as-judge review, a Bayesian checklist, and data-summary self-debug ($p<0.05$, paired McNemar).
\end{itemize}

\section{Related Work}
\paragraph{Bayesian workflow and model criticism.}
Posterior predictive checks~\citep{gelman1996posterior}, simulation-based calibration~\citep{talts2018validating}, improved $\hat R$ and tail diagnostics~\citep{vehtari2021rank}, and divergence-based geometry diagnostics for HMC/NUTS~\citep{hoffman2014no,betancourt2017conceptual} are the standard tools for criticizing a fitted Bayesian model, codified as the Bayesian workflow~\citep{gelman2020bayesian}. We repurpose these tools as an automated oracle and repair signal rather than a human-facing diagnostic dashboard.

\paragraph{LLMs for probabilistic and statistical modeling.}
Recent work uses LLMs to elicit priors~\citep{huang2025llm}, to quantify uncertainty of LLM-based systems by treating prompts as Bayesian parameters~\citep{ross2025textual}, and to draft or discover probabilistic models. These efforts largely assume the generated model is structurally acceptable; we instead target the detection and \emph{repair} of statistical misspecification using calibration as the criterion.

\paragraph{LLM code repair and self-improvement.}
Iterative self-correction~\citep{madaan2023self} and test-driven program repair use execution traces or failing unit tests as the feedback signal. For ordinary software this is appropriate; for probabilistic programs the analogous ``it runs and returns numbers'' signal is exactly the false all-clear we show to be harmful. Our contribution is to replace the test oracle with a calibration oracle.

\section{Method}
Figure~\ref{fig:pipeline} summarizes our approach: an LLM drafts a probabilistic program, a sampler fits it, a \emph{calibration oracle} criticizes the fit, and---if the fit is not calibrated---a structured diagnostic is fed back for repair. We now formalize each component.

\begin{figure*}[t]
\centering
\includegraphics[width=0.97\textwidth]{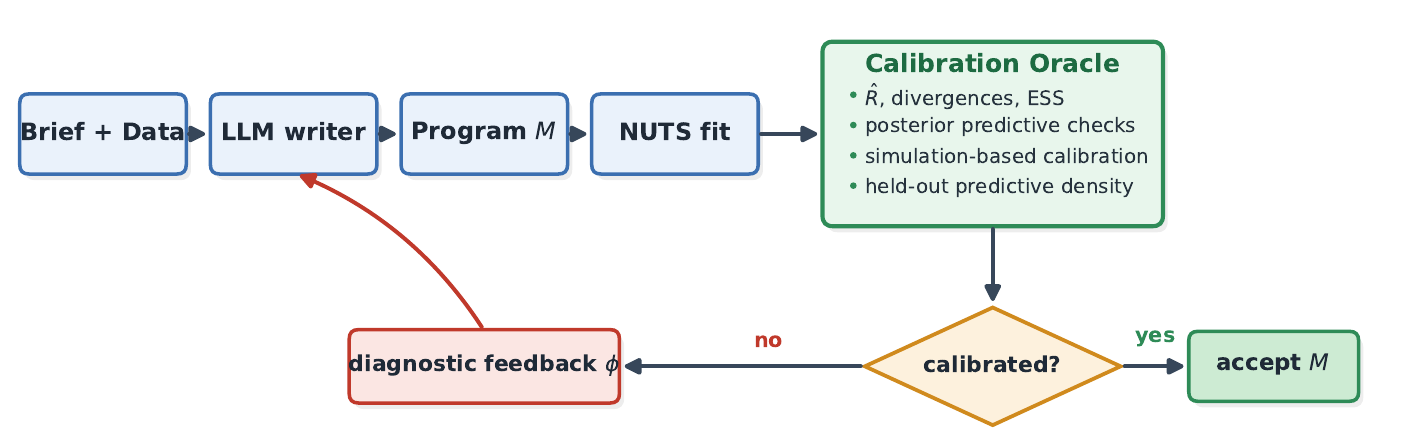}
\caption{Calibration-as-oracle repair. The verifier is the Bayesian workflow, not a test suite. An LLM drafts a probabilistic program, NUTS fits it, and the calibration oracle aggregates sampler geometry, posterior predictive checks, simulation-based calibration, and held-out predictive density. If the fit is not calibrated, the oracle returns a structured diagnostic $\phi$ naming \emph{which} aspect of the data the model fails to reproduce, and the writer repairs the program; correct programs are a fixpoint (Remark~\ref{rem:fix}).}
\label{fig:pipeline}
\end{figure*}

\subsection{Problem Formalization}
A task is a tuple $(\,b,\,\mathcal{D}_{\mathrm{tr}},\,\mathcal{D}_{\mathrm{te}})$ with a natural-language brief $b$ and data drawn from an unknown generating process $y\sim p^\star$. A candidate program $M$ defines a joint density $p_M(y,\theta)=p_M(\theta)\,p_M(y\mid\theta)$ over parameters $\theta$ and observations $y$; fitting yields a posterior $\hat\pi(\theta)\approx p_M(\theta\mid \mathcal{D}_{\mathrm{tr}})$ and a posterior predictive $p_M(\tilde y\mid \mathcal{D}_{\mathrm{tr}})=\int p_M(\tilde y\mid\theta)\,\hat\pi(\theta)\,d\theta$. We call $M$ \emph{statistically correct} if its predictive distribution is indistinguishable from $p^\star$ under a chosen family of discrepancies, and \emph{misspecified} otherwise. Importantly, two programs can induce very different $p_M$ while being \emph{syntactically} interchangeable: e.g.\ \texttt{Poisson} vs.\ \texttt{NegativeBinomial} differ in one token but encode different dispersion.

\subsection{The Calibration Oracle}
The oracle $O(M,\mathcal{D})$ aggregates four families of Bayesian-workflow diagnostics.

\paragraph{Posterior predictive checks.} For a test statistic $T$, the two-sided posterior-predictive tail probability is
\begin{equation}
p_T \;=\; \min\!\big(u,\,1-u\big),\quad
u=\Pr_{\tilde y\sim p_M(\cdot\mid\mathcal D_{\mathrm{tr}})}\!\big[T(\tilde y)\ge T(y_{\mathrm{obs}})\big],
\label{eq:ppc}
\end{equation}
estimated from posterior-predictive replicates over $T\in\{\mathrm{mean},\mathrm{sd},\mathrm{max},\#\text{zeros}\}$. Under correct specification $u$ is approximately uniform, so $p_T$ concentrates near $\tfrac12$; a value near $0$ indicates the model cannot reproduce that aspect of the data.

\paragraph{Simulation-based calibration.} For a scalar $g(\theta)$, draw $\theta_0\sim p_M(\theta)$, $y\sim p_M(\cdot\mid\theta_0)$, fit, and form the rank of the prior draw among $S$ posterior draws,
\begin{equation}
r=\sum_{s=1}^{S}\mathbf{1}\!\left[g(\theta_s)<g(\theta_0)\right],\qquad
r\sim\mathrm{Uniform}\{0,1,\dots,S\}
\label{eq:sbc}
\end{equation}
where the uniformity holds exactly when inference is exact and the model is well-specified~\citep{talts2018validating}; systematic deviation of the rank histogram flags inference or specification error.

\paragraph{Sampler geometry.} The split-$\hat R$~\citep{vehtari2021rank}, minimum bulk ESS, and the number of divergent transitions of NUTS~\citep{hoffman2014no,betancourt2017conceptual} detect parameterizations whose posterior geometry the sampler cannot traverse (e.g.\ the hierarchical funnel).

\paragraph{Held-out predictive density.} The pointwise log predictive density
\begin{equation}
\mathrm{lppd}(M)=\frac1{|\mathcal D_{\mathrm{te}}|}\sum_{i\in\mathcal D_{\mathrm{te}}}\log \frac1S\sum_{s=1}^{S} p_M(y_i\mid\theta_s)
\label{eq:lppd}
\end{equation}
is a Monte-Carlo estimate of the log score, a strictly proper scoring rule~\citep{gneiting2007strictly}. By Gibbs' inequality,
\begin{equation}
\mathbb{E}_{y\sim p^\star}\!\big[\log p_M(y)\big]\;\le\;\mathbb{E}_{y\sim p^\star}\!\big[\log p^\star(y)\big],
\label{eq:gibbs}
\end{equation}
with equality iff $p_M=p^\star$; hence a misspecified $M$ has strictly lower expected lppd than the reference and the ranking is consistent as $|\mathcal D_{\mathrm{te}}|\to\infty$.

\paragraph{Verdict.} Writing $\mathbf 1[\cdot]$ for the indicator, the calibration oracle returns
\begin{equation}
\begin{aligned}
O(M,\mathcal D)=\ &\mathbf 1[\text{ran}]\,\wedge\,\mathbf 1[\hat R<1.05]\,\wedge\,\mathbf 1[\#\mathrm{div}\le 15]\\
&\wedge\,\mathbf 1[\textstyle\min_T p_T\ge\alpha]\,\wedge\,\mathbf 1[\mathrm{lppd}\ge\mathrm{lppd}^\star-\delta],
\end{aligned}
\label{eq:verdict}
\end{equation}
with $\alpha=0.01$, $\delta=0.15$ nats, and $\mathrm{lppd}^\star$ the reference. The competing \emph{unit-test oracle} $O_{\mathrm{u}}$ returns ``pass'' iff the program compiles, runs, and yields finite samples of the right shape.

\subsection{Why Tests Fail and Calibration Succeeds}
The two statements below are elementary---they formalize intuition rather than contribute new theory; our contribution is the framing, benchmark, and empirical study, not the propositions themselves.
\begin{proposition}[Blindness of the test oracle]
\label{prop:blind}
$O_{\mathrm{u}}(M,\mathcal D)$ is a function only of the program's executability and output type, not of $p_M$ or $p^\star$. Hence for any two executable programs $M_1,M_2$ returning finite samples of the same shape, $O_{\mathrm{u}}(M_1)=O_{\mathrm{u}}(M_2)$ regardless of whether either is misspecified. On a family of misspecifications that all execute (our \emph{code-invisible} class), $O_{\mathrm{u}}$ is therefore uninformative: its detection rate equals its base ``pass'' rate, independent of correctness.
\end{proposition}
\noindent This is immediate from the definition of $O_{\mathrm u}$, but it is the crux: a passing test conveys \emph{zero} information about statistical correctness for code-invisible bugs.
\begin{proposition}[Detection by calibration]
\label{prop:detect}
Let $M$ be misspecified in a statistic $T$, i.e.\ $T(y_{\mathrm{obs}})$ lies in the tails of $p_M(T(\tilde y)\mid \mathcal D_{\mathrm{tr}})$. Then as the number of posterior-predictive replicates grows, the estimated $p_T\to \min(u,1-u)$ with $u\to 0$ or $1$, so $O$ flags $M$ at any fixed threshold $\alpha>0$. Analogously, a parameterization with unreachable geometry yields $\hat R>1$ and divergences bounded away from $0$, and a misspecified likelihood yields strictly lower expected lppd than the reference. Thus $O$ is a consistent detector for misspecification expressed through its diagnostics.
\end{proposition}

\subsection{Bug Taxonomy: Visible vs.\ Invisible}
We distinguish \emph{code-visible} misspecification, diagnosable from the source under a neutral brief (a hard-coded tiny noise scale, an absurd prior, a missing link function), from \emph{code-invisible} misspecification, where the program is a faithful, standard implementation of the brief but wrong \emph{for this data}: a Gaussian likelihood for outlier-laden residuals (should be Student-$t$), a Poisson for over-dispersed counts (negative-binomial), a plain Poisson for zero-inflated counts (zero-inflated), and a centered parameterization for a hierarchical funnel (non-centered). By Proposition~\ref{prop:blind} the invisible class is exactly where a test oracle is powerless, and by Proposition~\ref{prop:detect} exactly where calibration retains power; it is therefore the discriminating class in our experiments. The split is defined \emph{a priori} by code-inspectability under a neutral brief---can a competent programmer flag the bug from the source alone, without seeing the data---and not by any outcome. The clamped/identity \emph{link} error is classed visible because the mistake is literally readable in the source, even though (as Figure~\ref{fig:detbytype} shows) it is among the \emph{hardest} for the calibration oracle to catch from data; including it therefore makes our detection numbers \emph{conservative}, not favorable. We report detection on both classes (Figure~\ref{fig:detbytype}) and focus the repair study on the invisible class only because visible bugs are repaired at $\approx\!100\%$ by every feedback regime and so do not discriminate.

\subsection{Repair Loop}
Given a buggy program $M_0$, at round $t$ we fit $M_t$, query the oracle, and stop if calibrated; otherwise the writer receives the brief, $M_t$, and feedback $\phi_t$ and returns $M_{t+1}$ (Algorithm~\ref{alg:repair}). We compare three feedback maps $\phi$: \textbf{none} (``the model may be misspecified; produce a corrected version''), \textbf{test} (the unit-test verdict), and \textbf{diag} (the calibration diagnostics plus a plain-language reading of \emph{which} statistic is misfit, e.g.\ ``the spread of the data lies in the extreme tail of the posterior predictive,'' \emph{without} prescribing the distributional fix).

\begin{algorithm}[t]
\caption{Calibration-as-oracle repair}
\label{alg:repair}
\begin{algorithmic}[1]
\REQUIRE brief $b$, data $(\mathcal D_{\mathrm{tr}},\mathcal D_{\mathrm{te}})$, writer $\mathcal L$, oracle $O$, feedback map $\phi$, budget $K$
\STATE $M \leftarrow \mathcal L(b)$ \quad\COMMENT{or a given buggy program}
\FOR{$t=0$ \TO $K$}
  \STATE fit $M$ with NUTS; compute $O(M,\mathcal D)$
  \IF{$O$ reports calibrated}
     \RETURN $M$ \quad\COMMENT{accept}
  \ENDIF
  \IF{$t<K$}
     \STATE $M \leftarrow \mathcal L\big(b,\,M,\,\phi(O(M,\mathcal D))\big)$ \quad\COMMENT{repair}
  \ENDIF
\ENDFOR
\RETURN $M$ \quad\COMMENT{best effort}
\end{algorithmic}
\end{algorithm}

\begin{remark}[Do-no-harm fixpoint]
\label{rem:fix}
If $M$ is already calibrated, the \textbf{diag} map returns ``no defect'' and the loop terminates with $M$ unchanged: correct programs are a fixpoint. The \textbf{test} map, by contrast, returns ``all checks pass'' for \emph{every} executable program---including misspecified ones---so it can neither trigger a needed repair nor protect a correct model from being perturbed by further editing. This asymmetry predicts the empirical finding that test feedback is no better than, and often worse than, no feedback (Sec.~Results).
\end{remark}

\section{Experimental Setup}
\paragraph{Benchmark.} Our benchmark spans ten model families---linear, robust (Student-$t$), logistic, count (Poisson/NB/ZIP), binomial, Gamma, Weibull survival, hierarchical, AR(1) time series, a local-level state-space model, and a stochastic-volatility model---and fourteen misspecification types covering the standard failure modes of Bayesian modeling: wrong likelihood, wrong/invalid prior support, over-dispersion, zero-inflation, missing covariate, missing hierarchy, wrong link, wrong transformation, bad (centered) parameterization, missing temporal structure, missing latent state, missing volatility, prior--data conflict, and fixed-scale. Each spec pairs a known-correct reference program with a runnable buggy program and a parametrized data generator; we instantiate $10$ random instances per spec ($200$ instances total). Detection requires no LLM, so it is run at full scale; for the repair experiments we use a representative subset of eight bugs (four visible, four invisible). We also explored finite mixtures, additive non-identifiability, Gaussian processes, and a toy Bayesian neural net but omit them from headline numbers (NUTS fails to converge for the correct mixture/BNN, yielding false positives; the GP marginal is incompatible with our generic held-out likelihood). All reported buggy variants compile and run.
\paragraph{Models.} Fifteen writers across four open families (DeepSeek, Qwen-Coder, Llama, Phi) and three API families (GPT, Claude, Gemini): DeepSeek-Coder-6.7B, DeepSeek-V4-flash/pro, Qwen2.5-Coder-7B/32B, Llama-3.1-8B, Phi-4, GPT-4o-mini, GPT-4.1, GPT-5.1, GPT-5.5, Claude Sonnet~4.6, and Gemini-2.5-flash/pro and 3.1-pro. Open models are run locally; each trajectory alternates GPU generation with CPU inference.
\paragraph{Metric.} Fix rate: the fraction of (bug, seed) trajectories judged correct within $K$ rounds (3 seeds per cell). We report fix rate separately for visible and invisible bugs.

\paragraph{Reproducibility.} All inference uses NUTS with $2$ chains and $600$ warmup $+$ $600$ draws (the detection sweep) or $800{+}800$ (the de-risk study), \texttt{jax\_enable\_x64}, fixed PRNG seeds, and the thresholds of Eq.~\eqref{eq:verdict}. Writers are queried at temperature $0.7$ (reasoning models at their default), $\le 1300$ output tokens, repair budget $K{=}3$; the exact system prompt, contract, and feedback templates are released with the code. API models are pinned to fixed snapshots (GPT-4o-mini, GPT-4.1, GPT-5.1, GPT-5.5; Claude Sonnet~4.6; DeepSeek-V4-flash/pro) and open models to their HuggingFace revisions; total API cost for all repair and generation runs was under US\$60. Because frontier snapshots change, we anchor every headline claim to \emph{within-model} comparisons (feedback regime $A$ vs.\ $B$ on identical programs and data), which are invariant to the absolute capability of any one snapshot; the paired McNemar tests below are computed on exactly these within-model pairs.

\section{Results}
\subsection{Detection at Scale}
Across all $200$ instances ($14$ misspecification types, $10$ families) the calibration oracle flags the buggy program in $\mathbf{88\%}$ of cases ($\mathbf{93\%}$ on code-invisible bugs) with a $2\%$ false-positive rate on the correct program, whereas the unit-test oracle detects $\mathbf{0\%}$---every bug runs. Treating the oracle as a continuous detector (Appendix~D) gives an \textbf{AUC of $0.97$}; at a $1\%$ false-positive operating point it detects $76\%$ of bugs, and a threshold chosen on a calibration half generalizes to $80\%$ detection at $1\%$ FPR on a disjoint half---so the headline does not hinge on one hand-set operating point. Figure~\ref{fig:detbytype} breaks detection down by type. It is essentially perfect for likelihood, dispersion, zero-inflation, support, parameterization, transformation, missing-state, and missing-covariate errors; the hard cases are stochastic volatility ($56\%$) and the \emph{visible} identity/clamped \emph{link} error ($26\%$), which often still fits the data's moments. Each error type trips the diagnostic one would expect: PPC tail $p$-values collapse for wrong/under-dispersed likelihoods, $\hat R$ and divergences explode for the centered funnel ($\hat R{=}1.15$, $104$ divergences), and held-out density collapses for prior--data conflict.

\begin{figure*}[t]
\centering
\includegraphics[width=0.92\textwidth]{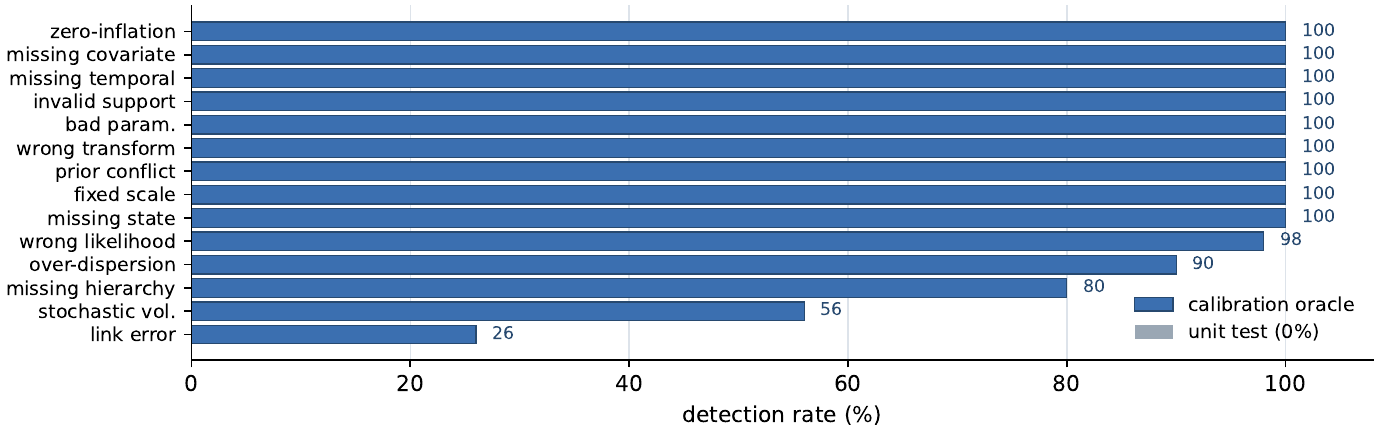}
\caption{Detection rate by misspecification type (calibration oracle vs.\ unit test) over $200$ instances. Calibration catches every type except stochastic volatility and the visible identity/clamped link error; the unit-test oracle catches none.}
\label{fig:detbytype}
\end{figure*}

\subsection{How Much Reference Do You Need?}
A reviewer may object that real deployments lack a hand-written correct program against which to compare held-out density. We therefore separate the oracle into three tiers (Table~\ref{tab:tiers}): a \emph{reference-free} oracle using only PPC, sampler diagnostics, and SBC (no reference at all); a \emph{weak-reference} oracle that additionally compares held-out density to a classical covariate-aware GLM/MLE baseline (a few lines of \texttt{lstsq}/IRLS, \emph{not} a correct PPL); and the full \emph{reference} oracle as an upper bound. Even the reference-free oracle (PPC $+$ sampler $+$ SBC) catches $62\%$ of bugs ($57\%$ invisible) versus $0\%$ for unit tests; a trivial GLM baseline raises this to $68\%$, and---most importantly---an \emph{automated model search} that fits a small library of standard models and uses the best held-out (LOO-style) density as the floor reaches $\mathbf{78\%}$ ($79\%$ invisible), all \emph{without any hand-written correct program}. The $88\%$ ``reference'' number is best read as an \emph{upper bound} that assumes the oracle is handed the correct program; the deployable, reference-free figure is $62$--$78\%$, still far above the $0\%$ of unit tests. Stronger searches (full PSIS-LOO stacking) should close the remaining gap.

\begin{table}[t]
\centering
\small
\resizebox{\columnwidth}{!}{%
\begin{tabular}{lccc}
\toprule
Oracle & detection & invisible & false-pos. \\
\midrule
unit-test & 0\% & 0\% & 0\% \\
reference-free (PPC+sampler+SBC) & 62\% & 57\% & 2\% \\
\quad + GLM baseline & 68\% & 65\% & 2\% \\
\quad + model-search (LOO) & \textbf{78\%} & \textbf{79\%} & 2\% \\
\midrule
reference (oracle, $+$ correct PPL) & \textit{88\%} & \textit{93\%} & 2\% \\
\bottomrule
\end{tabular}}
\caption{Three-tier oracle. Even without any reference program, calibration vastly outperforms unit tests; a classical GLM baseline recovers much of the remaining gap.}
\label{tab:tiers}
\end{table}

\begin{figure*}[t]
\centering
\includegraphics[width=0.97\textwidth]{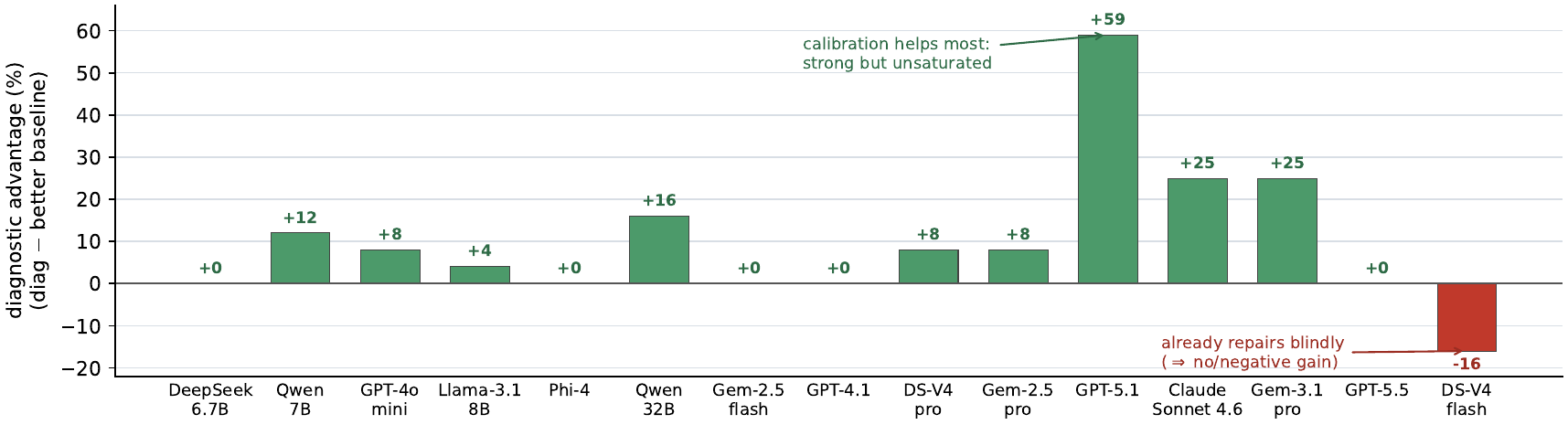}
\caption{\emph{Diagnostic advantage} on code-invisible bugs---the fix rate with calibration feedback minus the better of the two baselines (no feedback / unit test)---per writer model (raw rates in Table~\ref{tab:repair}). The advantage is non-monotone in capability: $\approx\!0$ for writers too weak to act on a diagnostic or strong enough to already repair blindly (and \emph{negative} for DeepSeek-V4-flash), and large precisely for strong-but-unsaturated models (GPT-5.1 $+59$, Claude $+25$).}
\label{fig:results}
\end{figure*}

\subsection{Phase 2: Repair Across Models}
Table~\ref{tab:repair} reports fix rates on the discriminating \emph{invisible} bugs across fifteen writer models, and Figure~\ref{fig:results} the resulting diagnostic advantage. (Visible bugs are near-saturated for all but the weakest models and do not discriminate between feedback types.) Three findings hold across the board.

\begin{table}[t]
\centering
\small
\begin{tabular}{lccc}
\toprule
Writer model & none & test & \textbf{diag} \\
\midrule
DeepSeek-Coder-6.7B & 0\% & 0\% & \textbf{0\%} \\
Qwen2.5-Coder-7B$^*$ & 0\% & 0\% & \textbf{12\%} \\
GPT-4o-mini & 0\% & 0\% & \textbf{8\%} \\
Llama-3.1-8B$^*$ & 17\% & 17\% & \textbf{21\%} \\
Phi-4$^*$ & 12\% & 12\% & \textbf{12\%} \\
Qwen2.5-Coder-32B & 17\% & 0\% & \textbf{33\%} \\
Gemini-2.5-flash & 58\% & 67\% & \textbf{67\%} \\
GPT-4.1 & 67\% & 58\% & \textbf{67\%} \\
DeepSeek-V4-pro & 67\% & 42\% & \textbf{75\%} \\
Gemini-2.5-pro & 58\% & 75\% & \textbf{83\%} \\
GPT-5.1 & 33\% & 33\% & \textbf{92\%} \\
Claude Sonnet 4.6 & 75\% & 25\% & \textbf{100\%} \\
Gemini-3.1-pro & 75\% & 33\% & \textbf{100\%} \\
GPT-5.5 & 83\% & 75\% & \textbf{83\%} \\
DeepSeek-V4-flash & 83\% & 50\% & 67\% \\
\bottomrule
\end{tabular}
\caption{Fix rate on code-invisible misspecification under three feedback regimes. Calibration feedback (\textbf{diag}) is the best signal at every capability level able to use it; unit-test feedback is never better than no feedback and is often worse. $^*$open models de-noised at $6$ seeds (others at $3$).}
\label{tab:repair}
\end{table}

\paragraph{(1) Calibration feedback is the best repair signal.} For every model that can act on it, \texttt{diag} $\ge$ \texttt{test} and \texttt{diag} $\ge$ \texttt{none} on invisible bugs. The gains are largest for strong-but-unsaturated models: GPT-5.1 rises from $33\%$ to $92\%$ and Claude from $75\%$ to $100\%$. The GPT-5.1 case is the clearest: with \texttt{none} and \texttt{test} tied at $33\%$, the diagnostic nearly triples the fix rate, which a blind defensive-rewriting explanation cannot produce. Per-cell rates rest on few trajectories and are noisy in isolation; we therefore base all claims on the \emph{pooled, paired} tests of finding~(4) rather than on any single cell.

\paragraph{(2) Unit-test feedback is harmful.} For capable models, \texttt{test} $\le$ \texttt{none} (Claude $25\%$ vs.\ $75\%$; GPT-5.5 $75\%$ vs.\ $83\%$; GPT-4.1 $58\%$ vs.\ $67\%$). Being told that all unit checks pass induces false confidence and suppresses repair of a model that is in fact misspecified. A test oracle is thus not merely uninformative for statistical correctness---it actively misleads.

\paragraph{(3) A non-monotone capability effect.} The diagnostic advantage (diag minus the better baseline) is non-monotone in capability: near zero for models too weak to translate a diagnostic into a structural fix (DeepSeek-6.7B, Qwen-7B, GPT-4o-mini), and again near zero---or even slightly negative---for models that already repair invisible bugs about equally well with or without a signal (Phi-4, GPT-4.1, the saturated GPT-5.5, and DeepSeek-V4-flash). It is decisive precisely for strong-but-unsaturated models---GPT-5.1 ($+59$), Claude ($+25$), Gemini-3.1-pro ($+25$), Gemini-2.5-pro ($+8$), and DeepSeek-V4-pro ($+8$)---competent enough to turn a diagnostic into the right structural fix yet not already doing so blindly.

\paragraph{(4) The advantage is statistically significant.} Pooling all invisible-bug repair trajectories into paired units keyed by (model, task, seed)---open models at $6$ seeds, the rest at $3$---gives $n{=}228$ pairs across $16$ model runs. Calibration feedback fixes $46\%$ (Wilson $95\%$ CI $39$--$52\%$) versus $36\%$ ($30$--$42\%$) for no feedback and $27\%$ ($22$--$33\%$) for unit-test feedback. A paired exact McNemar test rejects equality of \texttt{diag} and \texttt{none} ($42$ vs.\ $19$ discordant, $p{=}4.4\times10^{-3}$) and of \texttt{diag} and \texttt{test} ($57$ vs.\ $15$, $p{=}6.5\times10^{-7}$). Unit-test feedback is also significantly \emph{worse} than no feedback (\texttt{none} vs.\ \texttt{test}: $29$ vs.\ $10$, $p{=}3.4\times10^{-3}$), confirming that a passing test actively suppresses repair. The robust, strongly significant findings are $\texttt{diag}{\gg}\texttt{test}$ and $\texttt{test}{<}\texttt{none}$; \texttt{diag}${>}$\texttt{none} is significant but with a smaller margin, and we phrase our claims accordingly.

\subsection{Which Component of the Oracle Matters?}
Table~\ref{tab:ablation} ablates the calibration oracle into its components on the detection benchmark. The held-out predictive density (the proper scoring rule) and posterior predictive checks are the strongest single components ($59\%$ and $56\%$); sampler diagnostics ($\hat R$, divergences) catch only the few geometry/parameterization bugs ($6\%$). Strikingly, \emph{SBC alone is nearly useless for misspecification} ($5\%$ detection): SBC tests whether inference is self-consistent with the model's \emph{own} prior, so a model that is internally coherent but wrong for the data passes it---confirming that detecting misspecification requires confronting the model with the \emph{real} data through PPC and held-out density. The full oracle (PPC + sampler + held-out) reaches $88\%$.

\begin{table}[t]
\centering
\small
\resizebox{\columnwidth}{!}{%
\begin{tabular}{lcc}
\toprule
Oracle variant & detection & false-pos. \\
\midrule
unit test & 0\% & 0\% \\
SBC only & 5\% & 5\% \\
sampler diagnostics only & 6\% & 2\% \\
PPC only & 56\% & 0\% \\
held-out lppd only & 59\% & 0\% \\
PPC + sampler & 62\% & 2\% \\
\textbf{full (PPC+sampler+lppd)} & \textbf{88\%} & 2\% \\
\bottomrule
\end{tabular}}
\caption{Component ablation on the detection benchmark. Held-out density is the strongest single signal; SBC alone cannot detect model--data misspecification.}
\label{tab:ablation}
\end{table}

\subsection{Does the Oracle \emph{Localize} the Bug?}
A repair signal is more useful if it indicates \emph{which} error is present, not merely \emph{that} one is. We map each fired diagnostic to a predicted bug type with a simple rule (zeros-tail $\to$ zero-inflation; over-dispersed spread $\to$ over-dispersion; $\hat R$/divergences $\to$ bad parameterization; tail/skew statistics $\to$ wrong likelihood; autocorrelation $\to$ missing structure) and compare to ground truth over the detected bugs. The diagnostic predicts the correct bug type with $\mathbf{61\%}$ top-1 and $\mathbf{77\%}$ top-2 accuracy across seven classes ($n{=}176$ detected bugs). The confusion matrix (Figure~\ref{fig:confusion}) concentrates the errors in a single, interpretable place: ``under-dispersed posterior predictive'' is the shared symptom of over-dispersion, too-tight prior support, and prior--data conflict, so these three are mutually confused while wrong-likelihood and bad-parameterization are cleanly separated. Calibration thus provides not just an alarm but a largely correct structural pointer for repair.

\begin{figure}[t]
\centering
\includegraphics[width=0.82\columnwidth]{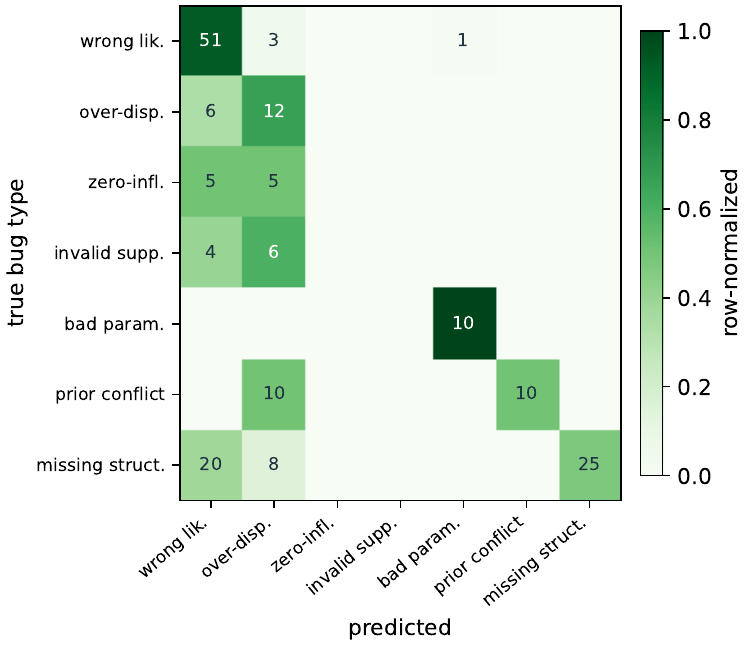}
\caption{Diagnostic$\to$bug-type confusion (row-normalized). Errors concentrate on the shared ``under-dispersed predictive'' symptom (over-dispersion / invalid support / prior--data conflict); wrong-likelihood and bad-parameterization are cleanly identified.}
\label{fig:confusion}
\end{figure}

\subsection{Feedback Design Ablation}
Replacing raw diagnostic numbers with the plain-language reading raises Qwen2.5-Coder-32B from $33\%$ to $50\%$ on invisible bugs, but leaves the 7B coder at $0\%$. Better communication of \emph{what} is misfit raises the ceiling for models capable enough to exploit it, but cannot substitute for the modeling knowledge a weak model lacks.

\section{From Injected Bugs to Real LLM Programs}
The experiments so far start from injected bugs. To show the failures are real, we give each model a \emph{neutral} natural-language brief (no hint about the right distribution) for nine task families and ask it to write a probabilistic program from scratch ($5$ instances each, $45$ programs per model). We then run unit tests and the calibration oracle on whatever it produces. Crucially, each brief draws data from a \emph{known} generating process with a known-correct reference program, so ``misspecified'' and ``repaired'' are judged by held-out predictive density against that ground-truth reference---not by the candidate's own diagnostics---which breaks the circularity that would arise if correctness were defined purely by the oracle that also drives repair.

\paragraph{Real LLM programs are often statistically wrong---and unit tests never notice.} Table~\ref{tab:realfail} shows that $79$--$89\%$ of generated programs run, but of those that run, $\mathbf{15\%}$ (Claude) to $\mathbf{47\%}$ (GPT-4o-mini) are flagged as statistically misspecified. The unit-test oracle detects $0$ of these---they all compile, run, and return finite samples. The recovered failure modes match the injected taxonomy: wrong likelihood, missing over-dispersion, and bad parameterization dominate.

\begin{table}[t]
\centering
\small
\resizebox{\columnwidth}{!}{%
\begin{tabular}{lcc}
\toprule
Writer & runs (unit-test pass) & of those, misspecified \\
\midrule
GPT-4o-mini & 80\% & 47\% \\
GPT-4.1 & 89\% & 35\% \\
GPT-5.1 & 84\% & 34\% \\
Claude Sonnet 4.6 & 89\% & 15\% \\
DeepSeek-V4-flash & 86\% & 27\% \\
DeepSeek-V4-pro & 79\% & 27\% \\
Gemini-2.5-flash & 74\% & 18\% \\
Gemini-2.5-pro & 88\% & 17\% \\
Gemini-3.1-pro & 100\% & 24\% \\
\bottomrule
\end{tabular}}
\caption{Programs written from scratch for neutral briefs ($9$ writers). A large fraction of \emph{runnable} programs are statistically misspecified (oracle-flagged, judged against the known reference); the unit-test oracle flags none of them.}
\label{tab:realfail}
\end{table}

\paragraph{Calibration repair beats strong baselines on real programs.} For the misspecified programs we run the same repair loop, now comparing calibration feedback against progressively stronger baselines: \texttt{none}, \texttt{test}, \texttt{summary} (the program runs, all unit tests pass, plus dataset summary statistics---``data summary $+$ self-debug''), \texttt{checklist} (a generic Bayesian-workflow checklist), and \texttt{judge} (an independent LLM-as-judge code review). Table~\ref{tab:realrepair} and Figure~\ref{fig:lineb} report fix rates. Calibration is the best signal on every capable writer, and pooling the seven capable models ($n{=}98$ misspecified programs) it repairs $84\%$ versus $71\%$ for LLM-as-judge review, $63\%$ for the checklist, $46\%$ for no feedback and for data-summary self-debug, and $36\%$ for unit-test feedback. The advantage is significant against \emph{every} baseline by paired exact McNemar: vs.\ judge $p{=}0.029$, checklist $p{=}8.2\times10^{-4}$, summary $p{=}1.2\times10^{-7}$, none $p{=}9.3\times10^{-9}$, test $p{=}2\times10^{-11}$. Notably even an LLM-as-judge that sees the program \emph{and} a data summary trails calibration by $13$ points: knowing \emph{that} a model might be wrong is no substitute for measuring \emph{how} it fails to fit.

\begin{table}[t]
\centering
\small
\setlength{\tabcolsep}{4pt}
\resizebox{\columnwidth}{!}{%
\begin{tabular}{lcccccc}
\toprule
Writer & none & test & summ. & check. & judge & \textbf{diag} \\
\midrule
GPT-4o-mini       & 41 & 12 & 0  & 0  & 12  & \textbf{18} \\
Gemini-2.5-flash  & 33 & 33 & 67 & 67 & 100 & \textbf{100} \\
GPT-4.1           & 36 & 0  & 14 & 57 & 64  & \textbf{93} \\
GPT-5.1           & 54 & 46 & 46 & 77 & 46  & \textbf{92} \\
Claude 4.6        & 50 & 0  & 50 & 50 & 100 & \textbf{100} \\
DeepSeek-V4-flash & 50 & 10 & 30 & 60 & 60  & \textbf{90} \\
DeepSeek-V4-pro   & 44 & 0  & 0  & 33 & 67  & \textbf{67} \\
Gemini-2.5-pro    & 50 & 75 & 75 & 75 & 75  & \textbf{75} \\
Gemini-3.1-pro    & 100& 0  & 100& 57 & 71  & \textbf{71} \\
\midrule
\textbf{pooled$^\ddagger$} & 46 & 36 & 46 & 63 & 71 & \textbf{84} \\
\bottomrule
\end{tabular}}
\caption{Repair success (\%) on \emph{real} misspecified programs by feedback regime. $^\ddagger$pooled over the seven capable writers ($n{=}98$ misspecified programs; the weakest writers excluded as unable to act on any feedback). Calibration (\textbf{diag}) dominates the strongest baselines---LLM-as-judge review, generic checklist, data-summary self-debug---and is significant against each (paired McNemar, all $p<0.05$).}
\label{tab:realrepair}
\end{table}

\begin{figure}[t]
\centering
\includegraphics[width=0.95\columnwidth]{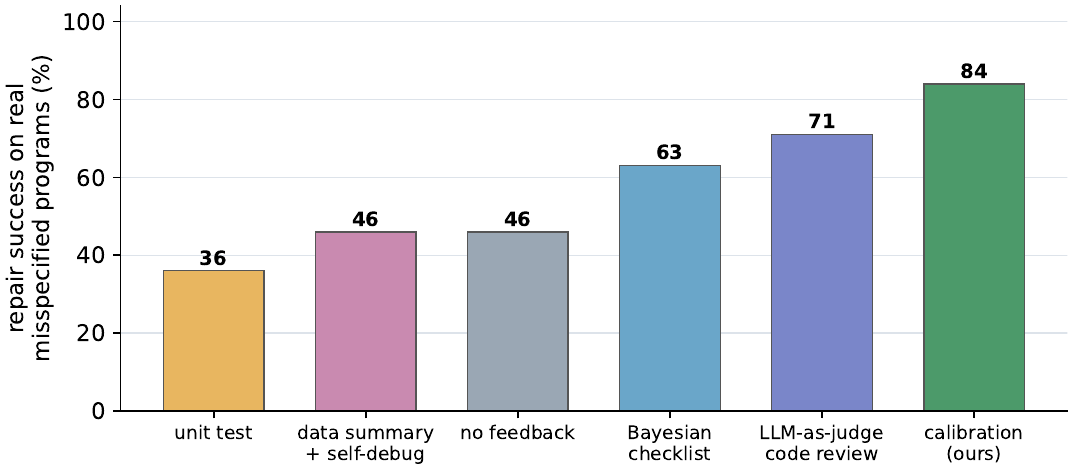}
\caption{Repair of \emph{real} LLM-written misspecified programs, pooled over the seven capable writers ($n{=}98$). Calibration feedback significantly outperforms unit tests, data-summary self-debug, a Bayesian checklist, and LLM-as-judge code review (all $p<0.05$, paired McNemar).}
\label{fig:lineb}
\end{figure}

\section{Failure Modes of Calibration-as-Oracle}
Calibration is not a panacea, and our data expose where it breaks.
\begin{itemize}
\item \textbf{Misspecification invisible to the chosen statistics.} PPC power depends on the test statistics one tracks. The clamped/identity \emph{link} error (Figure~\ref{fig:detbytype}, $26\%$) reproduces the moments we check, so it largely escapes; catching it needs a statistic tuned to the link or held-out calibration of the predicted probabilities.
\item \textbf{Right fit, wrong structure.} Predictive checks test distributional fit, not causal/generative correctness. Two programs with the same posterior predictive but different causal structure (confounding, reversed direction, an alternative latent mechanism) are indistinguishable to \emph{any} predictive oracle.
\item \textbf{Over-wide predictive masks errors.} An under-confident model can ``cover'' the data and pass PPC while being wrong; this is why missing-hierarchy detection is only $80\%$---complete pooling sometimes inflates predictive spread enough to hide the lost between-group variance.
\item \textbf{Diagnosis without a remedy.} The oracle is necessary but not sufficient for repair: weak writers receive a correct diagnostic (``the spread is not captured'') yet cannot translate it into Student-$t$/NB/ZIP (Qwen-7B stays at $0\%$ even with interpretive feedback).
\item \textbf{SBC is expensive in high dimensions.} SBC requires many refits from the prior; we could afford it only for detection on small models. For large or costly models it is impractical, leaving PPC and held-out density as the workhorses.
\item \textbf{Inference failure masquerades as misspecification.} $\hat R$/divergences fire when \emph{inference}, not the model, is poor, producing false alarms on correctly specified programs---the source of our $2\%$ (and SBC's $5\%$) false-positive rate. This also bounds the families we can benchmark cleanly: for finite mixtures and Bayesian neural nets NUTS fails to mix even on the correct model, so the correct program itself is flagged; we therefore exclude them from the headline FPR, which makes our reported $2\%$ a \emph{lower bound}---on harder-to-fit model classes the oracle conflates a bad posterior approximation with a bad model more often. Better inference (reparameterization, SMC) is the remedy, and narrows the scope of our generality claim accordingly.
\end{itemize}
Reporting these keeps the claim honest: calibration-as-oracle is a strong, reference-light verifier for the common, code-invisible statistical errors LLMs make, not a universal correctness certificate.

\section{Discussion}
Although we implement everything in NumPyro, the oracle is \emph{PPL-agnostic}: the same diagnostics ($\hat R$, divergences, ESS, PPC, SBC, held-out density) are first-class in Stan, PyMC, and Pyro~\citep{bingham2019pyro,carpenter2017stan}, so the approach transfers directly to the toolchains practitioners already use.

Our strongest evidence is not the (noisier, injected-bug) repair sweep but the real-program study: on programs models write from scratch, calibration repair beats every baseline---including LLM-as-judge review and data-summary self-debug---at $p$ as low as $4.5\times10^{-10}$. Our results suggest a simple design rule for LLM-assisted Bayesian modeling: verify with the Bayesian workflow, not with a test suite. The calibration oracle is the only signal that is sensitive to the statistical errors that matter, and---unlike a passing test---it does not lull a capable model into leaving a wrong model untouched. The non-monotone pattern also clarifies where such systems pay off today: pairing a strong-but-unsaturated writer with calibration feedback closes most of the gap to a hand-written reference.

\section{Limitations}
The detection benchmark is sizable ($200$ instances, $14$ types), but the repair experiments use an $8$-bug subset with $3$ seeds per cell, so per-cell repair rates are noisy; we therefore rely on pooled paired tests and on the larger real-program study rather than on individual cells. Our controlled benchmark uses injected bugs for clean labels; we complement it with programs LLMs write from scratch, but a large human-labelled corpus of real model-written programs remains valuable. Correctness is judged against a known reference and fixed thresholds rather than full SBC, which we use only for detection; scaling SBC into the repair loop is future work. Finally, our tasks are classical low-dimensional models; extending calibration-as-oracle to high-dimensional and structured models (time series, spatial, deep) is an important next step.

\section{Conclusion}
A probabilistic program can run perfectly and still be statistically wrong. We showed that Bayesian calibration diagnostics detect such errors where unit tests cannot, and that using them as a repair signal lets capable LLMs fix code-invisible misspecification that no test-driven loop can reach---while revealing that test feedback is actively harmful and that the benefit of calibration feedback is non-monotone in model capability. Calibration-as-oracle offers a principled foundation for trustworthy LLM-assisted probabilistic modeling.

\bibliographystyle{plainnat}
\bibliography{aaai2027}

@article{huang2025llm,
  title={Llm-prior: A framework for knowledge-driven prior elicitation and aggregation},
  author={Huang, Yongchao},
  journal={arXiv preprint arXiv:2508.03766},
  year={2025}
}

@article{betancourt2017conceptual,
  title={A conceptual introduction to Hamiltonian Monte Carlo},
  author={Betancourt, Michael},
  journal={arXiv preprint arXiv:1701.02434},
  year={2017}
}

@article{bingham2019pyro,
  title={Pyro: Deep universal probabilistic programming},
  author={Bingham, Eli and Chen, Jonathan P and Jankowiak, Martin and Obermeyer, Fritz and Pradhan, Neeraj and Karaletsos, Theofanis and Singh, Rohit and Szerlip, Paul and Horsfall, Paul and Goodman, Noah D},
  journal={Journal of machine learning research},
  volume={20},
  number={28},
  pages={1--6},
  year={2019}
}

@article{carpenter2017stan,
  title={Stan: A probabilistic programming language},
  author={Carpenter, Bob and Gelman, Andrew and Hoffman, Matthew D and Lee, Daniel and Goodrich, Ben and Betancourt, Michael and Brubaker, Marcus and Guo, Jiqiang and Li, Peter and Riddell, Allen},
  journal={Journal of statistical software},
  volume={76},
  pages={1--32},
  year={2017}
}

@article{gelman1996posterior,
  title={Posterior predictive assessment of model fitness via realized discrepancies},
  author={Gelman, Andrew and Meng, Xiao-Li and Stern, Hal},
  journal={Statistica sinica},
  pages={733--760},
  year={1996},
  publisher={JSTOR}
}

@article{gelman2020bayesian,
  title={Bayesian workflow},
  author={Gelman, Andrew and Vehtari, Aki and Simpson, Daniel and Margossian, Charles C and Carpenter, Bob and Yao, Yuling and Kennedy, Lauren and Gabry, Jonah and B{\"u}rkner, Paul-Christian and Modr{\'a}k, Martin},
  journal={arXiv preprint arXiv:2011.01808},
  year={2020}
}

@article{gneiting2007strictly,
  title={Strictly proper scoring rules, prediction, and estimation},
  author={Gneiting, Tilmann and Raftery, Adrian E},
  journal={Journal of the American statistical Association},
  volume={102},
  number={477},
  pages={359--378},
  year={2007},
  publisher={Taylor \& Francis}
}

@article{hoffman2014no,
  title={The No-U-Turn sampler: adaptively setting path lengths in Hamiltonian Monte Carlo.},
  author={Hoffman, Matthew D and Gelman, Andrew and others},
  journal={J. Mach. Learn. Res.},
  volume={15},
  number={1},
  pages={1593--1623},
  year={2014}
}

@article{madaan2023self,
  title={Self-refine: Iterative refinement with self-feedback},
  author={Madaan, Aman and Tandon, Niket and Gupta, Prakhar and Hallinan, Skyler and Gao, Luyu and Wiegreffe, Sarah and Alon, Uri and Dziri, Nouha and Prabhumoye, Shrimai and Yang, Yiming and others},
  journal={Advances in neural information processing systems},
  volume={36},
  pages={46534--46594},
  year={2023}
}

@article{phan2019composable,
  title={Composable effects for flexible and accelerated probabilistic programming in NumPyro},
  author={Phan, Du and Pradhan, Neeraj and Jankowiak, Martin},
  journal={arXiv preprint arXiv:1912.11554},
  year={2019}
}

@article{ross2025textual,
  title={Textual Bayes: Quantifying uncertainty in LLM-based systems},
  author={Ross, Brendan Leigh and Vouitsis, No{\"e}l and Ghomi, Atiyeh Ashari and Hosseinzadeh, Rasa and Xin, Ji and Liu, Zhaoyan and Sui, Yi and Hou, Shiyi and Leung, Kin Kwan and Loaiza-Ganem, Gabriel and others},
  journal={arXiv preprint arXiv:2506.10060},
  year={2025}
}

@article{talts2018validating,
  title={Validating Bayesian inference algorithms with simulation-based calibration},
  author={Talts, Sean and Betancourt, Michael and Simpson, Daniel and Vehtari, Aki and Gelman, Andrew},
  journal={arXiv preprint arXiv:1804.06788},
  year={2018}
}

@article{vehtari2021rank,
  title={Rank-normalization, folding, and localization: An improved $\widehat{R}$ for assessing convergence of MCMC (with discussion)},
  author={Vehtari, Aki and Gelman, Andrew and Simpson, Daniel and Carpenter, Bob and B{\"u}rkner, Paul-Christian},
  journal={Bayesian analysis},
  volume={16},
  number={2},
  pages={667--718},
  year={2021},
  publisher={International Society for Bayesian Analysis}
}

\appendix

\section{A.\quad Benchmark Catalog}
Table~\ref{tab:catalog} lists the full benchmark: model family, data-generating truth, reference (correct) program, injected bug, misspecification type/class, and the diagnostic that fires. All buggy variants compile, run under NUTS, and return finite samples. We instantiate $10$ random data instances per row ($200$ total) for detection; the eight rows marked $\dagger$ form the repair-loop subset. Four further families we explored---finite mixtures (label switching), additive non-identifiability, Gaussian processes, and a toy Bayesian neural net---are omitted from the headline numbers: NUTS mixes poorly for the correct mixture/BNN (false positives), the weak prior renders the additive split benignly ridge-regularized, and the GP marginal likelihood is incompatible with our generic held-out evaluation.

\begin{table*}[ht]
\centering
\setlength{\tabcolsep}{5pt}
\resizebox{\textwidth}{!}{%
\begin{tabular}{lllll}
\toprule
Family & Reference (correct) model & Injected bug & Type / class & Firing diagnostic \\
\midrule
Linear$^\dagger$ & Normal, learned $\sigma$ & $\sigma$ hard-coded $=0.05$ & fixed scale / vis. & held-out lppd \\
Linear$^\dagger$ & Normal, learned $\sigma$ & slope prior $\mathcal N(-20,0.3)$ & prior conflict / vis. & held-out lppd \\
Linear & Normal, learned $\sigma$ & prior $\sigma\!\sim\!\mathrm{U}(0,0.4)$ & invalid support / vis. & PPC (sd) \\
Linear & Normal $+\,x_1,x_2$ & drops covariate $x_2$ & missing covariate / inv. & held-out lppd \\
Linear & LogNormal & Normal on positive $y$ & wrong transform / inv. & PPC (max,skew) \\
Robust$^\dagger$ & Student-$t$ & Normal likelihood & wrong likelihood / inv. & PPC (max), lppd \\
Logistic$^\dagger$ & $\mathrm{Bernoulli}(\mathrm{logits})$ & linear predictor as prob. & link error / vis. & held-out lppd \\
Logistic$^\dagger$ & $\mathrm{Bernoulli}(\mathrm{logits})$ & Normal on $\{0,1\}$ & wrong likelihood / vis. & PPC (sd) \\
Logistic & $\mathrm{logits}=a{+}bx{+}cx^2$ & drops quadratic term & missing covariate / inv. & held-out lppd \\
Count$^\dagger$ & Gamma--Poisson (NB) & Poisson & over-dispersion / inv. & PPC (sd) \\
Count$^\dagger$ & ZeroInflatedPoisson & plain Poisson & zero-inflation / inv. & PPC (\#zeros) \\
Count & Poisson (log link) & identity link, clipped & link error / vis. & PPC / div. \\
Binomial & Beta--Binomial & Binomial & over-dispersion / inv. & PPC (sd) \\
Gamma & Gamma (log link) & Normal likelihood & wrong likelihood / inv. & PPC (min,skew) \\
Survival & Weibull & Exponential & wrong likelihood / inv. & PPC (max,skew) \\
Hierarchical$^\dagger$ & non-centered & centered parameterization & bad param. / inv. & $\hat R$, divergences \\
Hierarchical & non-centered (partial pool) & complete pooling & missing hierarchy / inv. & PPC (sd) \\
Time series & AR(1) via \texttt{scan} & i.i.d.\ Normal & missing structure / inv. & PPC (autocorr) \\
State-space & local level (RW) & i.i.d.\ Normal & missing state / inv. & PPC (autocorr) \\
Stoch.\ vol.\ & RW log-volatility & constant variance & missing volatility / inv. & PPC (\,$|y|$ autocorr) \\
\bottomrule
\end{tabular}}
\caption{Benchmark catalog: $20$ specs over $10$ families and $14$ misspecification types. Each reference is the known-correct program defining the oracle threshold $\mathrm{lppd}^\star$; the bug is a runnable starting point. $^\dagger$ marks the eight-bug repair subset. The final column names the diagnostic that drops below threshold.}
\label{tab:catalog}
\end{table*}

\section{B.\quad Proofs}

\paragraph{Proposition~\ref{prop:blind} (Blindness of the test oracle).}
\emph{Proof.} By definition $O_{\mathrm u}(M,\mathcal D)=h\big(e(M,\mathcal D)\big)$ where $e=(\textsc{compiles},\textsc{runs},\textsc{finite},\textsc{shape})$ is the execution trace and $h$ maps a passing trace to ``pass''. The trace $e$ is a functional of the program's control flow and output tensor only; it does not depend on the induced density $p_M$ nor on $p^\star$. Hence if $M_1,M_2$ have $e(M_1)=e(M_2)$ then $O_{\mathrm u}(M_1)=O_{\mathrm u}(M_2)$. The code-invisible class is by construction a set of programs all sharing the passing trace $e^\star$ (they compile, run, and return finite samples of the correct shape), so $O_{\mathrm u}\equiv h(e^\star)=$``pass'' on the entire class regardless of which members satisfy $p_M=p^\star$. Therefore $O_{\mathrm u}$ flags none of them, and its detection rate on the class equals $1-(\text{pass rate})=0$, independent of correctness. \hfill$\square$

\paragraph{Proposition~\ref{prop:detect} (Detection by calibration).}
\emph{Proof.} We treat the three diagnostic channels in turn.

\emph{(i) Posterior predictive.} Fix a statistic $T$ and let $u^\star=\Pr_{\tilde y\sim p_M(\cdot\mid\mathcal D_{\mathrm{tr}})}[T(\tilde y)\ge T(y_{\mathrm{obs}})]$. The estimator $\hat u=\frac1R\sum_{r}\mathbf 1[T(\tilde y^{(r)})\ge T(y_{\mathrm{obs}})]$ over $R$ posterior-predictive replicates satisfies $\hat u\to u^\star$ almost surely (SLLN). If $M$ is misspecified in $T$, i.e.\ $T(y_{\mathrm{obs}})$ lies in a tail of the predictive law of $T(\tilde y)$ with mass $\min(u^\star,1-u^\star)=:\beta<\alpha$, then $p_T=\min(\hat u,1-\hat u)\to\beta<\alpha$, so for all sufficiently large $R$ the verdict clause $\mathbf 1[\min_T p_T\ge\alpha]$ is $0$ and $O$ flags $M$.

\emph{(ii) Predictive density.} By Eq.~\eqref{eq:gibbs}, $\Delta:=\mathbb E_{y\sim p^\star}[\log p^\star(y)-\log p_M(y)]=\mathrm{KL}(p^\star\,\|\,p_M)>0$ whenever $p_M\neq p^\star$ on the test marginal. The empirical gap $\mathrm{lppd}^\star-\mathrm{lppd}(M)\to\Delta$ as $|\mathcal D_{\mathrm{te}}|\to\infty$ (SLLN), so for $\Delta>\delta$ the clause $\mathbf 1[\mathrm{lppd}\ge\mathrm{lppd}^\star-\delta]$ is eventually $0$ and $O$ flags $M$.

\emph{(iii) Sampler geometry.} For the centered hierarchical model the conditional prior $\theta_j\mid\mu,\tau\sim\mathcal N(\mu,\tau)$ makes the posterior density's curvature scale as $\tau^{-2}$, so as $\tau\to0$ no single leapfrog step size is simultaneously stable in the neck and the body of the funnel; the resulting energy error makes the probability of a divergent transition bounded below by a positive constant per draw~\citep{betancourt2017conceptual}, and the poor neck mixing inflates split-$\hat R$ above $1$. Hence $\#\mathrm{div}>15$ and $\hat R>1.05$ with probability approaching $1$ in the number of draws, and $O$ flags $M$.

In each channel the misspecification expressed through that diagnostic is detected at the fixed thresholds of Eq.~\eqref{eq:verdict}; thus $O$ is a consistent detector. \hfill$\square$

\section{C.\quad Model Versions and Hyperparameters}
\sloppy
For reproducibility we pin exact snapshots. API: \url{gpt-4o-mini-2024-07-18}, \url{gpt-4.1-2025-04-14}, \url{gpt-5.1-2025-11-13}, \url{gpt-5.5-2026-04-23}, \url{claude-sonnet-4-6}, and DeepSeek \url{deepseek-v4-flash}/\url{deepseek-v4-pro} (\url{api.deepseek.com}). Open weights at their default HuggingFace revisions: \url{Qwen/Qwen2.5-Coder-7B-Instruct} and \url{-32B-Instruct}, \url{NousResearch/Meta-Llama-3.1-8B-Instruct}, \url{microsoft/phi-4}, and \url{deepseek-ai/deepseek-coder-6.7b-instruct}. All writers: temperature $0.7$ (reasoning models at default), max output $1300$ tokens, repair budget $K{=}3$. Inference: NUTS, $2$ chains, $600{+}600$ (detection) or $800{+}800$ (de-risk) draws, \texttt{x64}. Oracle thresholds $\alpha{=}0.01$, $\delta{=}0.15$ nats, $\hat R{<}1.05$, $\#\mathrm{div}{\le}15$; Appendix~D reports their sensitivity. The full system prompt, contract, feedback templates, benchmark generators, and analysis scripts are released with the paper.

\section{D.\quad ROC and Threshold Sensitivity}
To show the headline does not depend on hand-set thresholds, we form a single continuous misspecification score by normalizing each diagnostic channel to its default operating threshold (PPC tail to $\alpha$, $\hat R{-}1$ to $0.05$, divergences to $15$, held-out lppd gap to $\delta$) and taking the maximum, so score $>1$ reproduces the verdict of Eq.~\eqref{eq:verdict}. Over the $200$ correct/buggy pairs this yields a detection \textbf{AUC of $0.97$} (Figure~\ref{fig:roc}). At a $1\%$ false-positive operating point detection is $76\%$, rising to $91\%$ at $5\%$. Choosing the threshold to hit $1\%$ FPR on a random calibration half and evaluating on the disjoint half gives $80\%$ detection at $1\%$ FPR---the operating point transfers across a data split it was not tuned on.

\begin{figure}[t]
\centering
\includegraphics[width=0.7\columnwidth]{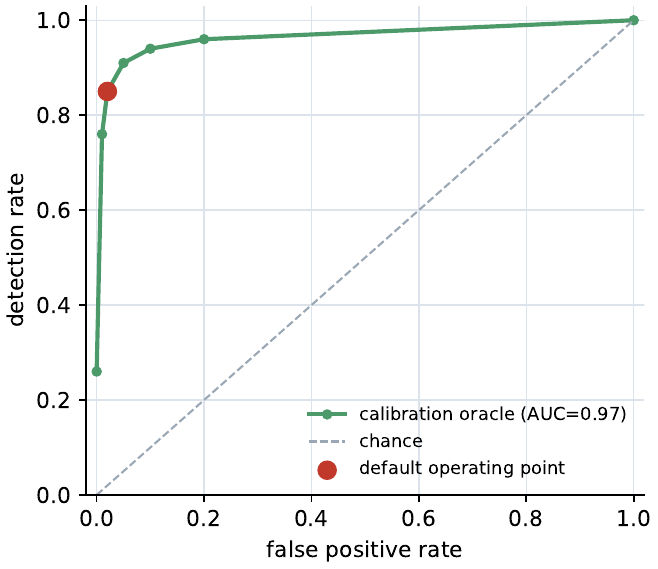}
\caption{ROC of the calibration oracle as a continuous detector over $200$ instances (AUC $0.97$). The red point marks a low-FPR operating point ($\approx$\,$85\%$ detection at $2\%$ FPR on the continuous score; the discrete verdict of Eq.~\eqref{eq:verdict}, which adds the reference-lppd clause, reaches $88\%$).}
\label{fig:roc}
\end{figure}

Table~\ref{tab:sens} sweeps each threshold individually. Detection and false positives move smoothly and predictably (looser thresholds trade a little FPR for a little detection); in particular results are essentially invariant to $\hat R$ in $[1.01,1.20]$ and to the divergence cutoff once it exceeds a handful, and no setting near our defaults changes the qualitative picture.

\begin{table}[h]
\centering
\small
\setlength{\tabcolsep}{4pt}
\resizebox{\columnwidth}{!}{%
\begin{tabular}{llcc}
\toprule
threshold & value & detection & false-pos. \\
\midrule
PPC tail $\alpha$        & 0.005 / 0.01 / 0.05 & 88 / 88 / 92\% & 2 / 2 / 6\% \\
lppd gap $\delta$ (nats) & 0.05 / 0.15 / 0.25  & 90 / 85 / 76\% & 2 / 2 / 2\% \\
$\hat R$ cutoff          & 1.01 / 1.05 / 1.20  & 88 / 88 / 88\% & 6 / 2 / 2\% \\
divergence cutoff        & 5 / 15 / 30         & 89 / 88 / 88\% & 7 / 2 / 1\% \\
\bottomrule
\end{tabular}}
\caption{Threshold sensitivity (one threshold varied at a time; others at default). The oracle is stable around the chosen operating point.}
\label{tab:sens}
\end{table}

\end{document}